\documentclass[12pt]{article}
\usepackage{amssymb,amsthm,amsmath}
\usepackage{algorithm} 
\usepackage{algorithmic} 

\newtheorem{lemma}{Lemma}[section]

\newtheorem{definition}[lemma]{Definition}

\newtheorem{theorem}[lemma]{Theorem}
\newtheorem{remark}[lemma]{Remark}

\begin{document}
\title{Evolution is Still Good: Theoretical Analysis of Evolutionary Algorithms on General Cover Problems}
\author{Yaoyao Zhang$^1$, Chaojie Zhu$^2$, Shaojie Tang$^3$, Ringli Ran$^2$, \\Ding-Zhu Du$^4$, Zhao Zhang$^2$\footnote{Corresponding author: Zhao Zhang, hxhzz@sina.com}\\
 {\small $^1$ College of Mathematics and System Science, Xinjiang
University}\\
 {\small Urumqi, Xinjiang, 830046, China}\\
 {\small $^2$ College of Mathematics and Computer Science, Zhejiang Normal University}\\
 {\small  Jinhua, Zhejiang, 321004, China}\\
 {\small $^3$ Naveen Jindal School of Management, University of Texas at Dallas}\\
 {\small Richardson, Texas, 75080, USA}\\
 {\small $^4$ Department of Computer Science, University of Texas at Dallas}\\
 {\small Richardson, Texas, 75080, USA}}
\date{}
\maketitle

\begin{abstract}
Theoretical studies on evolutionary algorithms have developed vigorously in recent years. Many such algorithms have theoretical guarantees in both running time and approximation ratio. Some approximation mechanism seems to be inherently embedded in many evolutionary algorithms. In this paper, we identify such a relation by proposing a unified analysis framework for a generalized simple multi-objective evolutionary algorithm (GSEMO), and apply it on a minimum weight general cover problem. For a wide range of problems (including the the minimum submodular cover problem in which the submodular function is real-valued, and the minimum connected dominating set problem for which the potential function is non-submodular), GSEMO yields asymptotically tight approximation ratios in expected polynomial time.

\end{abstract}


\section{Introduction}\label{sec1}

Evolutionary algorithms are heuristic search methods  inspired by biological evolution \cite{Vikhar}.
Although evolutionary algorithms have long been verified to be effective and efficient in empirical studies,  rigorous analyses about these algorithms did not emerge until the late 1990s. Considerable progress has been made in the theoretical understanding and analysis of evolutionary algorithms  \cite{DoerrBook,NeumannBook,Zhoubook}. In particular, it is interesting to see that evolutionary algorithms have good approximation guarantees as well as good running times for many NP-hard problems \cite{Friedrich,Friedrich1,Oliveto,Qian3,Qian4,Qian1,Qian2,Yu}. We notice that most of existing studies reply on the property of submodularity of their utility functions,  and those studies which do not require their utility functions to be submodular also depend on parameters measuring how far their utility functions are from submodularity. It remains largely open whether evolutionary algorithms can still achieve good approximation ratios in the absence of submodularity.  We also notice that previous studies on multi-objective evolutionary algorithms mainly focus on integral constraints. It is not clear whether real-valued constraints can be dealt with efficiently. In this paper, we aim to (partially) fill this gap by investigating the performance bounds of evolutionary algorithms for a broad class of minimum general cover problems, whose utility functions might be real-valued, and are not necessarily submodular. A formal definition of this class of general cover problems is as follows:

\begin{definition}[minimum general cover problem (MinGC)]\label{def:1}
{\rm Suppose $X=\{v_1,...,v_n\}$ is an element set, $w:X\mapsto \mathbb R^+$  is a \emph{weight (or cost) function} on $X$, and $g:2^X\mapsto \mathbb R^+$, called the \emph{utility function}, is a  real-valued set-monotone-nondecreasing function. The MinGC problem is to find a set $C\subseteq V$ satisfying
\begin{align}\label{eq0313-1}
\min_{C\subseteq X} & \ \sum_{x\in C}w(x)\\
\mbox{s.t.} & \  g(C)=g(X)  \nonumber
\end{align}}
\end{definition}

Note that MinGC is general enough to subsume many important problems, including the {\em minimum connected dominating set problem} (MinCDS) and the {\em  minimum submodular cover problem} (MinSubmC), as special cases.

In this paper, we develop a new general purpose analytical framework, called {\em multi-phase bin-tracking analysis} (MultiBinTrack), to derive a performance bound that does not rely on the approximate-submodularity for non-submodular utility functions. We develop this technique progressively, explaining the rationale behind the design of each component, and apply this technique to analyze the performance bound of {\em global simple evolutionary multi-objective optimizer} (GSEMO) \cite{Giel0}, a simple and classic evolutionary algorithm, for the MinGC problem. The basic idea of this technique  is to build a connection between an evolutionary algorithm and a greedy algorithm, which select items recursively based on their marginal utility-to-cost ratios.

It is well known that greedy strategy often has very good performance bounds for many coverage problems. A major contribution of this paper is to  derive sufficient conditions under which GSEMO can achieve nearly the same approximation ratio as the greedy algorithm. Both MinCDS and integer-valued MinSubmC satisfy these conditions, hence GSEMO yields bounded approximation ratios  for both problems in expected polynomial time.  Furthermore, our framework gives a bi-criterion approximation algorithm for the {\em real-valued} MinSubmC problem violating the feasibility constraint by a small additive factor.
It should be clarified that our main contribution is on developing a new technique to analyze GSEMO (which is an existing general purpose evolutionary algorithm) for the MinGC problem, rather than inventing new algorithms, with an attempt to reveal deeper approximation mechanism underlying this evolutionary algorithm.

\subsection{Related Works}\label{sec1.1}

Since the end of the last century, considerable progress has been made in understanding the theoretical performance of evolutionary algorithms \cite{DoerrBook,NeumannBook,Zhoubook}. In the following, we only examine performance bounds of evolutionary algorithms for those most closely related coverage problems.

One classic example of such problems is the {\em minimum set cover} (MinSC) problem. Friedrich {\it et al.} \cite{Friedrich} showed that an approximation ratio of $(\ln n+1)$ can be achieved by a {\em global simple evolutionary multi-objective optimizer} (GSEMO) in expected time $O (n^2m + mn(\log m + \log c_{max}))$, where $n$ is the number of elements to be covered, $m$ is the number of sets, and $c_{max}$ is the maximum cost of a set. For the $k$-MinSC problem, in which every set has size at most $k$, \cite{Yu} introduced a framework of evolutionary algorithm which yields an approximation ratio that can be achieved by a centralized approximation algorithm developed in \cite{Levin}. The {\em minimum vertex cover problem} (MinVC) is a special case of the MinSC problem, which has been a focus of many theoretical studies on evolutionary algorithms, including running time versus approximation ratio \cite{Friedrich,Oliveto}, FPT algorithms \cite{Gao,Kratsch,Pourhassan}, and in the dynamic setting \cite{Pourhassan2,Pourhassan1,ShiF}.

Coverage function is a special submodular function. Various submodular optimization problems, due to their wide applications in artificial intelligence, have recently attracted a lot of attention from researchers studying theoretical aspects of evolutionary algorithms,  especially on submodular maximization under a cardinality constraint \cite{Qian3,Qian4,Qian2}. Friedrich and Neumann further studied the submodular maximization under matroid constraints \cite{Friedrich1}.

There had been several attempts to apply evolutionary algorithms to non-submodular optimization problems \cite{Qian3,Qian4,Qian2}. They mostly focus on the (utility) maximization problem rather than the (cost) minimization problem as studied in this paper. Moreover, they often use a parameter called \emph{submodularity ratio},  or similar concepts which may be called \emph{approximate-submodularity}, to bound the distance of a non-submodular function to a submodular function. As a result, their performance bounds depend on the value of the approximate-submodularity.

Note that the approximate-submodularity of the utility functions of many MinGC problems such as the MinCDS problem could be arbitrarily large, making existing solutions ineffective for the MinGC problem. Furthermore, in the above submodular or non-submodular optimization problems, the constraints are integer-valued. It is not clear whether good approximation ratios can be achieved by evolutionary algorithms when the constraints are real-valued, such as the real-valued MinSubmC problem.

 The remaining part of this paper is organized as follows. In Section \ref{sec3}, we give an overview of  the technique of multi-phase bin-tracking analysis for GSEMO. In Section \ref{sec2}, we apply this technique to analyze the performance of GSEMO on the MinGC problem, and further apply the results to two special MinGC problems: the integer-valued minimum general cover problem (which includes the MinCDS problem) and the real-valued MinSubmC problem. Section \ref{sec4} concludes the paper and discusses future work.

\section{Overview of Algorithm Design and Analysis}\label{sec3}

In this section, we give a brief introduction to GSEMO, which is a classic general-purpose evolutionary algorithm. Then we give an overview of the technique of multi-phase bin-tracking analysis  (MultiBinTrack).

\subsection{Overview of  GSEMO}\label{sec3.1}

We first introduce some notations. A subset $S\subseteq X=\{v_1,\ldots,v_n\}$ can be identified with its {\em characteristic vector} $\textbf{x}\in \{0,1\}^n$, in which the $i$-th bit $x_i=1$ if and only if $v_i\in S$. In the following, we do not distinguish a vector ${\bf x}$ and the set it represents, and use terminology {\em individual} to refer to them. Consider a minimization problem with a bi-objective function $(f_1({\bf x}), f_2({\bf x}))$, 
we say an individual $\textbf{x}^\prime$  {\em weakly dominates} $\textbf{x}$, denoted as $\textbf{x}^\prime\succeq \textbf{x}$, if $f_i(\textbf{x}^\prime)\leq f_i(\textbf{x})$ holds for any $i\in\{1, 2\}$. In this case, we also say that  ${\bf x}'$ is {\em weakly better} than ${\bf x}$, or ${\bf x}$ is {\em weakly inferior} to ${\bf x}'$. We say that $\textbf{x}^\prime$  {\em dominates} $\textbf{x}$, denoted as $\textbf{x}^\prime\succ\textbf{x}$, if $\textbf{x}^\prime\succeq \textbf{x}$ and there exists an $i\in\{1, 2\}$ with $f_i(\textbf{x}^\prime)<f_i(\textbf{x})$. In this case, we also say that ${\bf x}'$ is {\em better than} ${\bf x}$, or ${\bf x}$ is {\em inferior} to ${\bf x}'$. If neither $\textbf{x}^\prime\succeq \textbf{x}$ nor $\textbf{x}\succeq \textbf{x}^\prime$, then $\textbf{x}$ and $\textbf{x}^\prime$ are {\em incomparable}. 

A typical multi-objective evolutionary algorithm maintains a \emph{population} $P$, which is composed of a set of individuals that are mutually incomparable, i.e., $\forall\: \textbf{x}\in P$, the set $\{\textbf{x}^\prime\in P, \textbf{x}^\prime\succeq \textbf{x}, \textbf{x}^\prime\neq \textbf{x}\}=\emptyset$. 
It starts with some initial population $P=\{{\bf x}_0\}$. In each iteration, an individual ${\bf x}$ is picked uniformly at random from $P$, and mutated into an offspring ${\bf x}'$.  If ${\bf x}'$ is not inferior to any individual in $P$, then ${\bf x}'$ is added into $P$ and those individuals which are weakly inferior to ${\bf x}'$ are deleted from $P$. There are various ways of performing the mutation. In GSEMO \cite{Giel0}, which is the focus of this paper, it flips every bit of ${\bf x}$ independently with probability $1/n$.

As discussed earlier, our main focus is on developing a novel technique to theoretically analyze GSEMO for the MinGC problem rather than inventing new algorithms.

\subsection{Multi-Phase  Bin-Tracking Analysis}\label{sec3.2}

In this section, we give an overview of MultiBinTrack.
Consider a problem of minimizing a bi-objective function $(f_1({\bf x}),f_2({\bf x}))$ such that $f_1(\cdot)$ is an utility function which measures feasibility, and $f_2(\cdot)$ is a cost function. Suppose $f_1({\bf x})$ takes values from a discrete set $\{\xi_0,\xi_1,\ldots,\xi_{\beta}\}$, where $\xi_0<\xi_1<\cdots<\xi_{\beta}$, and ${\bf x}$ is a feasible solution if and only if $f_1({\bf x})=\xi_0$. Our goal is to find a feasible solution to minimize $f_2({\bf x})$. 

As GSEMO progresses, we maintain a group of $\beta+1$ bins $B_{\xi_0}, B_{\xi_1},\ldots, B_{\xi_{\beta}}$, each of which is empty initially.
By abusing notation a little without ambiguity, for each  $i\in\{0,1,\ldots,\beta\}$, we use the same notation  $B_{\xi_i}$  to refer to the $i$-th bin as well as the set of individuals contained in that bin.  It is important to clarify that this bin system is created only for the purpose of analysis, the implementation of GSEMO does not rely on this system. Once a new individual ${\bf x}'$ is generated and inserted into $P$ by GSEMO,  we add ${\bf x}'$ to $B_{f_1({\bf x}')}$ and delete all individuals that are weakly inferior to  ${\bf x}'$ from the bin system if and only if
${\bf x}'$ satisfies some {\em quality control condition} $\pi$.
Note that any individual ${\bf x}$ in $B_{\xi_i}$ has $f_1({\bf x})=\xi_i$,  and $B_{\xi_0}\neq\emptyset$ implies that a feasible solution has been reached.   However, a feasible solution might not be good. To ensure that an individual that can be put into $B_{\xi_0}$ has a good quality, the key is to find appropriate conditions $\pi$ to restrict those individuals that can enter the bin system.

To bound the running time, we introduce a \emph{tracker} $I$, which tracks the smallest index of the non-empty bin, i.e.,  $I=\min\{i\in \{0,1,\ldots,\beta\}\colon B_{\xi_i}\neq \emptyset\}$.  The analysis of time complexity involves two factors:

$(a)$ prove that $I$ does not increase with more individuals added into the bin system;

$(b)$ starting from an arbitrary stage of GSEMO with $I=i$, estimate the expected time for $I$ to decrease by at least 1, denote this expected time as $l_i$.

\noindent Then the expected time it takes for $I$ to reach $0$, which indicates that we have successfully found a feasible solution in $B_{\xi_0}$, is at most $\sum_{i=1}^{\beta}l_i$.


It turns out that the above framework of analysis is general enough to subsume the analysis used in many existing studies, including  the maximum matroid base problem \cite{Qian1,Reichel}, the minimum set cover problem \cite{Friedrich,Yu}, the maximum submodular optimization problem \cite{Friedrich1,Qian3,Qian4}, as special cases. For example, consider the {\em minimum set cover problem} (MinSC). Given a set of $n$ elements $E$ and a collection of $m$ subsets $\mathcal S\subseteq 2^E$, each set $S\in\mathcal S$ has a positive cost $c(S)$, the goal of MinSC is to select a minimum cost subcollection $\mathcal F\subseteq \mathcal S$ to cover all elements, i.e. $\bigcup_{S\in\mathcal F}S=E$ and the cost $c(\mathcal F)=\sum_{S\in\mathcal F}c(S)$ is the minimum. The analysis in \cite{Friedrich} for the $H_n$-approximate evolutionary algorithm for MinSC (where $H_n=\sum_{i=1}^n1/i$ is the $n$-th Harmonic number) can be restated using the above framework as follows. Let $f_1({\bf x})$ be the number of uncovered elements under ${\bf x}$ and $f_2({\bf x})=c({\bf x})$.
 Then $f_1({\bf x})$ can only take discrete values $0,1,\ldots,n$, and $f_1({\bf x})=0$ indicates that ${\bf x}$ corresponds to a set cover.  For each $i\in\{0, 1,\ldots,n\}$ and any individual ${\bf x}$ with $f_1({\bf x})=i$, we say that ${\bf x}$ satisfies condition $\pi$ if and only if $c({\bf x})\leq (H_n-H_{i})opt$, where $opt$ is the optimal value (note that the optimal value $opt$ is only used for the purpose of analysis). This indicates that every individual from $B_0$ is an $H_n$-approximate set cover. It can be shown that the tracker $I$ is monotone non-increasing, and the expected time it takes for $I$ to decrease by at least one is upper bounded by $O(mn)$, hence, the expected time to find an $H_n$-approximate solution is $O(n^2m)$, given that the starting population is $\emptyset$.

Unfortunately, when applying the above bin-tracking analysis to the  MinGC problem, which subsumes the   MinCDS problem and the real-valued MinSubmC problem  as special cases, we encounter additional challenges.

For the   MinCDS problem,   we found that the aforementioned bin-tracking analysis only works when $f_1({\bf x})$ is relatively large,   that is, when ${\bf x}$ is relatively far from feasible.  This motivates us to extend the  bin-tracking analysis to  {\em multi-phase}  bin-tracking analysis. In the analysis, we conduct the bin-tracking analysis in multiple phases, and in each phase, we adopt a different quality-control condition. This enables us to handle the case when $f_1({\bf x})$ is small. One challenge to be conquered in this multi-phase analysis is how to concatenate different phases in a smooth manner. We introduce the concept of ``advance'' to ensure quality control conditions and a smooth concatenation of different phases. These problems will be elaborated
in Section \ref{sec3.3}, where we apply this technique to analyze the performance of GSEMO on the MinGC problem.

Additional efforts are required to deal with the {\em real-valued} MinSubmC problem, as the feasibility function takes values from a {\em continuous} range. 


\section{Solving MinGC}\label{sec2}

In this section, we apply MultiBinTrack to analyze the performance of GSEMO for the MinGC problem (Definition \ref{def:1}). It is assumed that $g(\cdot)$ is real-valued, normalized ($g(\emptyset)=0$) and monotone nondecreasing, but is not necessarily submodular.

In Section \ref{sec0626-1}, we design a greedy algorithm for the MinGC problem and give sufficient conditions under which this algorithm can achieve a theoretically guaranteed approximation ratio. Then in Section \ref{sec3.3}, we show how to use MultiBinTrack to analyze the performance of GSEMO on MinGC.   In Section \ref{sec0308-1}, we apply the results to some special cases of the MinGC problem, showing that GSEMO can achieve almost the same approximation ratios as that of greedy algorithms in expected polynomial time.

\subsection{Greedy Algorithm}\label{sec0626-1}

In this section, we present a greedy algorithm {\sc Greedy} for the MinGC problem.  The algorithm uses a greedy strategy, choosing a most cost-effective element in each iteration. What is different from the other works is the sufficient conditions that we formulate for the algorithm to work for the MinGC problem with a theoretically guaranteed approximation ratio. For two subsets $S, S'\subseteq X$, let $\Delta_Sg(S')=g(S\cup S')-g(S')$ be the {\em marginal profit} of $S$ over $S'$.  {\sc Greedy} starts with an initial solution  $C=\emptyset$. In each subsequent iteration $t$, {\sc Greedy} adds to $C$ an element $b$ satisfying
\[b=\arg\max_v \{\Delta_vg(C)/w(v): v\in X\setminus C\},\]
where $\Delta_vg(C)/w(v)$ is called the {\em cost-effectiveness }of element $v$, w.r.t. $C$.
This process iterates until a feasible solution is reached. A detailed implementation of {\sc Greedy} is described in  Algorithm \ref{algo1}.

\begin{algorithm} [H]
\caption{{\sc Greedy}}

\begin{algorithmic}[1]\label{algo1}
\STATE \textbf{Input:} A MinGC instance $(X,w,g)$.
\STATE \textbf{Output:} A subset $C\subseteq X$ which is a feasible solution to MinGC.
\STATE $C\leftarrow \emptyset$
\WHILE{  $g(C)<g(X)$}
    \STATE $b\leftarrow \arg\max_v \{\Delta_vg(C)/w(v): v\in X\setminus C\}$  \label{line4}
    \STATE $C\leftarrow C\cup \{b\}$
\ENDWHILE
\RETURN $C$
\end{algorithmic}
\end{algorithm}

Assume, without loss of generality, that the  weight of the cheapest element is $1$, and let $w_{\max}=\max_{v\in X}w(v)$ denote the weight of the most expensive element.  The following parameter $\delta$ will be used in analyzing the approximation ratio:

\begin{equation}\label{eq0810-1}
\delta=\min_{C\subset X,v\in X\setminus C,g(C\cup\{v\})>g(C)}\{g(C\cup\{v\})-g(C)\}.
\end{equation}
Intuitively, $\delta$ measures the {\em degree of sparsity}, that is, the smallest gap between two distinct values of $g$. We always use $C^*$ to denote an optimal solution and let $opt=w(C^*)$.

\begin{theorem}\label{thm1}
Suppose a MinGC instance has the following properties:
\begin{itemize}
\item[$(\romannumeral1)$] for any element set $C\subset X$ with $g(C)<g(X)$, there is an element $v\in X\setminus C$ such that $\Delta_vg(C)>0$, and
\item[$(\romannumeral2)$] there exists a constant $p$  such that for any $C\subset X$,  the elements in $C^*\setminus C$ can be ordered as $v_1,\ldots v_t,v_{t+1},\ldots, v_{\widehat{t}} $, where $t$ is the smallest index satisfying $g(C_t^*\cup C)= g(X)$, here $C^*_{i}=\{v_1,\ldots,v_i\}$ and $C^*_{0}=\emptyset$, and for all $i=1,\dots, t$,
\begin{align}\label{eq0728-0}
\Delta_{v_{i}}g(C^*_{i-1}\cup C)\leq \Delta_{v_{i}}g(C)+p.
\end{align}
\end{itemize}
Then {\sc Greedy} achieves an approximation ratio of at most $(p+1)\frac{w_{\max}}{\delta}+\ln\frac{g(X)-p\cdot opt}{opt}$  (if $g(X)-(p+1)\cdot opt\leq0$, then $\ln\frac{g(X)-p\cdot opt}{opt}$ is viewed as 0).
\end{theorem}

\begin{proof}  Suppose the output of Algorithm \ref{algo1} is $C=\{b_1,\ldots,b_s\}$ where $b_i$ is the element selected in the $i$-th iteration. Denote $C_0=\emptyset$ and $C_i=\{b_1,\ldots,b_i\}$  for each $i\in\{1,\ldots,s\}$. For $i\in\{1,\ldots, s\}$, let $\{v_{i,1},\ldots,v_{i,t_i},v_{i,t_i+1},\ldots,v_{i,\hat t_i}\}$ be the ordered set of $C^*\setminus C_{i-1}$ satisfying the condition of this theorem, where $t_i$ is the smallest index satisfying $g(C_{i-1}\cup \{v_{i,1},\ldots,v_{i,t_i}\})= g(X)$. Denote $C_{i,j}^*=\{v_{i,1},\ldots,v_{i,j}\}$ for $j=1,\ldots,t_i$, and let $C_{i,0}^*=\emptyset$.

 By the definition of $\delta$ and the greedy choice of $b_i$, we have $\Delta_{b_i}g(C_{i-1})\geq \delta$ for any $i\in\{1,\ldots,s\}$. Hence
 \begin{align}\label{eq0730-1}
\frac{\Delta_{b_i}g(C_{i-1})}{w(b_i)}\geq \frac{\delta}{w_{\max}}.
\end{align}
If $g(X)\leq (p+1)opt$, then combining \eqref{eq0730-1} with  $g(C_s)=g(X)$ and $g(\emptyset)=0$, we have
\begin{align*}
w(C_s)&=\sum_{i=1}^sw(b_i) \leq \sum_{i=1}^s\frac{w_{\max}}{\delta}\Delta_{b_i}g(C_{i-1}) \\
&=\frac{w_{\max}}{\delta}(g(C_s)-g(\emptyset))=\frac{w_{\max}}{\delta}g(X)    \\ &\leq\frac{w_{\max}}{\delta}\big(p+1)opt,
\end{align*}
and the desired approximation ratio holds in this case. In the following we assume
\begin{align}\label{eq0314-10}
g(X)> (p+1)opt.
\end{align}

For  $i\in\{0,1,\dots,s\}$, let $\alpha_i=g(X)-g(C_i)-p\cdot opt$.  The following claim shows that $\alpha_i$ decreases geometrically if $\alpha_i$ is nonnegative.

\textbf{Claim 1.} If $\alpha_{i-1}>0$, then
\begin{align}\label{eq0314-7}
\alpha_i\leq  e^{-\frac{w(b_i)}{opt}}\alpha_{i-1}.
\end{align}

Consider the $i$-th iteration. By the greedy choice of $b_i$, we have
\begin{align*}
\frac{\Delta_{b_i}g(C_{i-1})}{w(b_i)}\geq \frac{\Delta_{v_{i,j}}g(C_{i-1})}{w(v_{i,j})},\ \forall j\in\{1,\ldots,t_i\}.
\end{align*}
It follows that
\begin{align}\label{eq0314-1}
\frac{\Delta_{b_i}g(C_{i-1})}{w(b_i)}\geq \frac{\sum_{j=1}^{t_i}\Delta_{v_{i,j}}g(C_{i-1})}{\sum_{j=1}^{t_i}w(v_{i,j})}.
\end{align}
Because  the minimum weight  is $1$, we have
\begin{align}\label{eq0314-4}
t_i\leq |C^*\setminus C_{i-1}|\leq w(C^*\setminus C_{i-1})\leq opt.
\end{align}
Combining \eqref{eq0728-0}, \eqref{eq0314-1}, \eqref{eq0314-4} and the fact that $\sum_{j=1}^{t_i}w(v_{i,j})\leq opt$, we have
\begin{align}
\frac{\Delta_{b_i}g(C_{i-1})}{w(b_i)} & \geq \frac{\sum_{j=1}^{t_i}\big(\Delta_{v_{i,j}}g(C_{i-1}\cup C_{i,j-1}^*)-p\big)}{opt}\nonumber \\
&=\frac{\sum_{j=1}^{t_i}\left(g(C_{i-1}\cup C_{i,j}^*)-g(C_{i-1}\cup C_{i,j-1}^*)\right) -p\cdot t_i}{opt}\nonumber\\
&\geq\frac{g(C_{i-1}\cup C_{i,t_i}^*)-g(C_{i-1})-p\cdot opt}{opt}\nonumber\\
& = \frac{g(X)-g(C_{i-1})-p\cdot opt}{opt}.\label{eq0314-5}
\end{align}
It follows that $\alpha_i=g(X)-g(C_i)-p\cdot opt$ satisfies
\begin{equation}\label{eq0627-1}
\frac{\alpha_{i-1}-\alpha_i}{w(b_i)}\geq \frac{a_{i-1}}{opt},
\end{equation}
and thus
$$
\alpha_i\leq \left(1-\frac{w(b_i)}{opt}\right)\alpha_{i-1}\leq e^{-\frac{w(b_i)}{opt}}\alpha_{i-1},
$$
where the second inequality uses the fact $1+x\leq e^x$. Claim 1 is proved.

Recursively using inequality \eqref{eq0314-7}, as long as $\alpha_{i-1}>0$, we have
\begin{align}\label{eq0314-8}
\alpha_i \leq e^{-\frac{\sum_{j=1}^iw(b_j)}{opt}}\alpha_0.
\end{align}

Note that assumption \eqref{eq0314-10} guarantees $\alpha_0>opt$. Since $\alpha_s=-p\cdot opt\leq opt$, there is an index $i_0$ such that $\alpha_{i_0}> opt$ and $\alpha_{i_0+1}\leq opt$. Let $w(b_{i_0+1})=d_1+d_2$ satisfy the following constraint:
\begin{align}\label{eq0315-1}
\frac{\alpha_{i_0}-opt}{d_1}=\frac{\alpha_{i_0}-\alpha_{i_0+1}}{w(b_{i_0+1})}=\frac{opt-\alpha_{i_0+1}}{d_2}.
\end{align}

{\bf Claim 2.} For the above index $i_0$, the following inequalities hold:
\begin{align}
& \sum_{i=1}^{i_0}w(b_i)+d_1\leq \ln\frac{\alpha_0}{opt}\cdot opt.\label{eq0314-9}\\
& \sum_{i=i_0+2}^sw(b_i)\leq \frac{w_{\max}}{\delta}\big(g(C_s)-g(C_{i_0+1})\big).\label{eq0314-12}\\
& d_2\leq \frac{w_{\max}}{\delta}(opt-\alpha_{i_0+1}).\label{eq0315-10}
\end{align}

Combining the first equality of \eqref{eq0315-1} with inequality \eqref{eq0627-1} (taking $i=i_0+1$), we have
$$
opt\leq \left(1-\frac{d_1}{opt}\right)\alpha_{i_0}\leq e^{-\frac{d_1}{opt}}\alpha_{i_0}.
$$
Combining this inequality with \eqref{eq0314-8}, we have
$$
opt\leq e^{-\frac{\sum_{j=1}^{i_0}w(b_j)+d_1}{opt}}\alpha_0.
$$
Then inequality \eqref{eq0314-9} follows by recollecting the terms.

By inequality \eqref{eq0730-1}, we have
$$
\sum_{i=i_0+2}^sw(b_i)\leq \frac{w_{\max}}{\delta}\sum_{i=i_0+2}^s\Delta_{b_i}g(C_{i-1})=\frac{w_{\max}}{\delta}\big(g(C_s)-g(C_{i_0+1})\big).
$$
Inequality \eqref{eq0314-12} is proved.
Inequality \eqref{eq0315-10} follows from the combination of the second equality of \eqref{eq0315-1}, \eqref{eq0730-1} and the fact that $\alpha_{i_0}-\alpha_{i_0+1}=\Delta_{v_{i_0+1}}g(C_{i_0})$. Claim 2 is proved.

Combining Claim 2 with the facts $g(C_s)= g(X)$ and $\alpha_{i_0+1}=g(X)-g(C_{i_0+1})-p\cdot opt$,
\begin{align*}
w(C_g) & =\sum_{i=1}^sw(b_i)= \sum_{i=1}^{i_0}w(b_i)+d_1    +d_2+ \sum_{i=i_0+2}^sw(b_i)\\
& \leq \ln\frac{\alpha_0}{opt}\cdot opt+\frac{ w_{\max}}{\delta}(opt-\alpha_{i_0+1})+\frac{w_{\max}}{\delta}\big(g(C_s)-g(C_{i_0+1})\big)\\
&= \ln\frac{\alpha_0}{opt}\cdot opt+\frac{w_{\max}}{\delta}\big(g(X)-g(C_{i_0+1})+opt-\alpha_{i_0+1}\big)\\
& =\left(  \frac{w_{\max}}{\delta}(1+p) +  \ln\frac{\alpha_0}{opt}\right)opt.
\end{align*}
The desired approximation ratio is proved.
\end{proof}
%

\subsection{GSEMO on MinGC} \label{sec3.3}

In this section, we apply MultiBinTrack to show that under the same conditions of Theorem \ref{thm1}, GSEMO achieves almost the same approximation ratio for the MinGC problem in expected polynomial time. A detailed   GSEMO  for MinGC is described in Algorithm \ref{algo2}. Given an instance of MinGC, the {\em fitness} of a solution  ${\bf x}$ is captured by a bi-objective function $(f_1({\bf x}), f_2({\bf x}))$, where $f_1({\bf x})$ measures the uncovered portion by ${\bf x}$ and  $f_2({\bf x})$ denotes the weight of ${\bf x}$.   Specifically, let $S_\textbf{x}$ denote the subset of elements corresponding to its characteristic vector $\textbf{x}$, we define
\begin{align}\label{eq0704-1}
& f_1(\textbf{x})=\left\lfloor\frac{g(X)-g(\textbf{x})}{\delta} \right\rfloor\cdot\delta \ \mbox{and}\\
& f_2(\textbf{x})=w(\textbf{x}),\nonumber
\end{align}
where $g(\textbf{x})=g(S_\textbf{x})$ and $w(\textbf{x})=\sum_{x\in S_\textbf{x} }w(x)$.

GSEMO starts from an empty population, i.e., initial $P=\{{\bf 0}\}$. In each subsequent iteration, it picks an individual ${\bf x}$ uniformly at random from $P$, and generates a new individual $\textbf{x}^\prime$ by flipping each bit of $\textbf{x}$  with probability $\frac{1}{n}$. We add ${\bf x}'$ into $P$ if $\textbf{x}^\prime$ is not inferior to any individual in $P$. If ${\bf x}'$ has been added to $P$, then we remove all individuals which are weakly inferior to ${\bf x}'$ from $P$. On termination, the algorithm outputs a best feasible solution stored in current $P$.  It should be noted that usually, an evolutionary algorithm starts from a randomly generated initial solution. We let the algorithm start from the zero solution in order to focus on the most central part of the analysis. In fact, by an analysis similar to that in \cite{Friedrich}, we can show that a zero solution can enter the population in expected polynomial time. It should also be pointed out that usually, an evolutionary algorithm runs infinitely. But we prefer setting a termination time for GSEMO. As we shall show latter, setting the termination time $T$ properly, a performance guaranteed solution can be obtained with high probability.


\begin{algorithm} [H]
\caption{{\sc GSEMO}}

\begin{algorithmic}[1] \label{algo2}
\STATE \textbf{Input:} $(X,f_1,f_2)$ with $|X|=n$ and the number of iterations $T$.

\STATE \textbf{Output:} an individual $\textbf{x}$.
\STATE  $P\leftarrow \{{\bf 0}\}$
\FOR{$t=1,2,\dots,T$}
 \STATE Select $\textbf{x}$ from $P$ uniformly at random;\label{line0606-1}
 \STATE Generate $\textbf{x}^\prime$ by flipping each bit of $\textbf{x}$  with probability $\frac{1}{n}$;\label{line5}
 \IF {($\nexists~\textbf{z}\in P$ with $\textbf{z}\succ \textbf{x}^\prime)$}\label{line0531-1}
  \STATE $P= P\setminus\{{\textbf{z}: \textbf{x}^\prime\succeq \textbf{z}}\wedge\textbf{z}\in P \}\cup \{\textbf{x}^\prime\}$\label{line0602-1}
 \ENDIF
\ENDFOR
\STATE Return  $\arg\min_{{\bf x}}\{f_2({\bf x})\colon {\bf x}\in P,f_1({\bf x})=0\}$, if any.
\end{algorithmic}
\end{algorithm}


%

\begin{remark}\label{rem1}
{\rm If we simply use $f_1({\bf x})=g({\bf x})$ in GSEMO, then when $g(\cdot)$ is a real-valued function, there might be too many individuals entering $P$, and thus the time/space complexity might not be bounded. Hence we discretize the function $g({\bf x})$ into $f_1({\bf x})$ as in \eqref{eq0704-1}. As a result, the solution returned from Algorithm 2 might violate the feasibility constraint by an additive error up to $\delta$. That is, an individual ${\bf x}$ satisfying $g({\bf x})> g(X)-\delta$ (or equivalently, $f_1({\bf x})=0$) is regarded as a nearly feasible solution. The goal is to find a {\em nearly feasible} solution to minimize $f_2({\bf x})$.  Note that when $g(\cdot)$ is integer-valued, we may take $\delta=1$, and there is no loss in feasibility.}
\end{remark}
\begin{remark}\label{rem2}
{\rm If the condition described in Theorem \ref{thm1} is satisfied,  then function $f_1(\cdot)$ satisfies the following inequality:
\begin{align*}
-\Delta_{v_{i}}f_1(C^*_{i-1}\cup C)&=  \left \lfloor\frac{g(X)-g(C^*_{i-1}\cup C)}{\delta}\right \rfloor\delta-\left \lfloor\frac{g(X)-g(C^*_{i-1}\cup C \cup \{v_i\})}{\delta}\right \rfloor\delta     \nonumber \\
&\leq \big(g(X)-g(C^*_{i-1}\cup C) \big)-\left(\frac{g(X)-g(C^*_{i-1}\cup C\cup \{v_i\})}{\delta}-1\right)\delta \nonumber \\
&=g(C^*_{i-1}\cup C\cup \{v_i\})-g(C^*_{i-1}\cup C)+\delta  \nonumber \\
&\leq \Delta_{ v_i }g(C)+p+\delta \nonumber \\ &= \big(g(X)-g(C)\big)-\big(g(X)-g(C\cup \{v_i\})\big) +p+\delta\nonumber \\
&\leq   \left(\left \lfloor\frac{g(X)-g(C}{\delta}\right \rfloor\delta+\delta \right )-\left \lfloor\frac{g(X)-g(C)\cup \{v_i\})}{\delta}\right \rfloor\delta+p+\delta\nonumber \\
&=-\Delta_{ v_i }f_1(C)+p+2\delta.
\end{align*}
 Note that if $g(\cdot)$ is integer-valued, then there is no loss of $2\delta$ in the above inequality.}
\end{remark}

The next lemma estimates the number of bins used in the analysis.
\begin{lemma}\label{lem0606-3}
Let $\beta=\left  \lfloor\frac{g(X)-g(\emptyset)}{\delta} \right  \rfloor $.  The population $P$ maintained by GSEMO (Algorithm \ref{algo2}) satisfies $|P|\leq \beta+1$ throughout the evolutionary process.
\end{lemma}
\begin{proof}
 By the monotonicity of $g(\cdot)$, we have $0=g({\bf 0})\leq g({\bf x})\leq g(X)$ for any individual ${\bf x}$. Then by the definition of $f_1(\cdot)$ in \eqref{eq0704-1}, $0\leq f_1({\bf x})\leq \left \lfloor (g(X)-g({\bf 0}))/\delta\right \rfloor\cdot\delta=f_1({\bf 0})=\beta\delta$. Since $f_1({\bf x})$  can only take values from $\{0, \delta,\ldots,\beta\delta\}$, it has at most $\beta+1$ possible values. According to line \ref{line0602-1} of Algorithm \ref{algo2}, for each $i\in \{0, \delta,\ldots,\beta\delta\}$, $P$ contains at most one individual whose $f_1$-value is $i$.  Hence, the size of $P$ is at most $\beta+1$.
\end{proof}


\begin{theorem}\label{thm2}
If a MinGC instance satisfies those conditions described in Theorem \ref{thm1}, then in expected $O(\beta^2 n)$ time, GSEMO returns a nearly feasible solution with approximation ratio of at most $\frac{w_{\max}}{\delta}\big(p+1+2\delta)+\ln\frac{f_1(\emptyset)-(p+2\delta )\cdot opt}{opt-\delta}$.  Furthermore, in the case when $g(\cdot)$ is integer-valued, GSEMO returns a feasible solution in $O(g(X)^2n)$ time that has approximation ratio of at most $w_{\max}(1+p) +  \ln\frac{g(X)-p\cdot opt}{opt}$.
\end{theorem}

\begin{proof} 
In the following we will apply the technique of MultiBinTrack as introduced in Section \ref{sec3.2} to prove this theorem. First create $\beta+1$ bins: $B_0, B_1, \cdots, B_{\beta}$, each of which is empty.  Initially, $P={\bf 0}$. Add the initial solution ${\bf 0}$  to  $B_{\beta}$. Recall that we use $I$ to track the index of the non-empty bin which has the smallest index. Hence, $I=\beta $ initially. Once a new individual ${\bf x}$ is generated, suppose $\frac{f_1({\bf x})}{\delta}=i $, add ${\bf x}$ to $B_{i}$ if and only if ${\bf x}$ ``advances'' some existing individual in the bin system (the meaning of advance and related operations on the bin system will be clarified latter). We divide the process into two phases. The first phase of analysis aims to bound the expected time it takes for $I$ to drop from $\beta $ to some value less than $\left \lfloor\frac{(p+2\delta+1)opt}{\delta}\right \rfloor$, and the second phase of analysis aims to bound the the expected time it takes for $I$ to further drop to  $0$. It should be emphasized again that $opt$ is only used for the purpose of analysis. Next, we explain these two phases of analysis in details.

\textbf{In the first phase of analysis,} we say that an  individual ${\bf x}$ satisfies a quality-control condition $\pi^{(1)}$ if
\begin{align}
f_1({\bf x})\leq \alpha^\prime_0e^{-\frac{f_2(\textbf{x})}{opt}}+(p+2\delta)\cdot opt,\label{eq0602-12}
\end{align}
where $\alpha^\prime_0=f_1(\emptyset)-(p+2\delta )\cdot opt$.  We say that ${\bf x}'$ {\em advances} ${\bf x}$ in the first phase if either ${\bf x}'\succeq {\bf x}$ or the following two conditions are satisfied:
\begin{align}
& f_1(\textbf{x}^\prime)-(p+2\delta)\cdot opt \leq \left (1-\frac{f_2(\textbf{x}^\prime)-f_2(\textbf{x})}{opt}\right)\big(f_1(\textbf{x})-(p+2\delta)\cdot opt \big)
\label{eq0316-2-2}\\
& \mbox{and}\ f_1(\textbf{x}^\prime)\leq f_1(\textbf{x})-\delta\ \mbox{and}\ f_2(\textbf{x}^\prime)\leq f_2(\textbf{x})+w_{\max} \label{eq0316-2}
\end{align}
Note that if ${\bf x}'\succeq {\bf x}$, then ${\bf x}$ and ${\bf x}'$ satisfy condition \eqref{eq0316-2-2}. Also note that condition \eqref{eq0316-2} says that ${\bf x}'$ might be inferior to ${\bf x}$ in terms of $f_2$-value, but the gap is no larger than $w_{\max}$, at the same time, ${\bf x}'$ must be strictly better than ${\bf x}$ in terms of $f_1$-value by an additive amount at least $\delta$.

 A newly generated individual ${\bf x}'$ {\em has a potential} to be added into the bin system if ${\bf x}'$ advances some existing individual ${\bf x}$ in the bin system. If furthermore, ${\bf x}'$ is eligible to enter the population $P$, then we add ${\bf x}'$ to $B_{i}$, where $i=f_1({\bf x}')/\delta$.  In order to be consistent with the population $P$, those individuals in the bin system that are deleted from $P$ because of the entering of ${\bf x}'$ are also deleted from the bin system.
If ${\bf x}'$ has the above potential but is not eligible to enter $P$, then there is an individual ${\bf y}\in P$ with ${\bf y}\succ {\bf x}'$. In this case ${\bf y}$ advances ${\bf x}$, and we add ${\bf y}$ to $B_{j}$, where $j=f_1({\bf y})/\delta$.
 The above manipulation ensures that an individual can enter the bin system only when it advances some existing element in the bin system, and individuals kept in the bin system also belong to the population. Note that the bin system only records those ``good'' individuals in $P$ for analysis.

First, we prove that the advance criterion \eqref{eq0316-2-2} can guarantee condition $\pi^{(1)}$.

{\bf Claim 1.} If ${\bf x}'$ is added to the bin system in the first phase, then ${\bf x}'$ satisfies $\pi^{(1)}$.

This claim can be proved by induction. Initially, $P=\{{\bf 0}\}$, and ${\bf 0}$ trivially satisfies $\pi^{(1)}$. When ${\bf x}'$ is added into the bin system, ${\bf x}'$ must advance some existing individual ${\bf x}$ in the bin system. Suppose ${\bf x}\in B_i$. By induction hypothesis, ${\bf x}$ satisfies property $\pi^{(1)}$, that is, inequality \eqref{eq0602-12}. Combining this with condition \eqref{eq0316-2-2}, we have
\begin{align*}
f_1({\bf x}') & \leq \left (1-\frac{f_2(\textbf{x}^\prime)-f_2(\textbf{x})}{opt}\right)\big(f_1(\textbf{x})-(p+2\delta)\cdot opt \big)+(p+2\delta)\cdot opt \\
& \leq e^{-\frac{f_2(\textbf{x}^\prime)-f_2(\textbf{x})}{opt}} \big(f_1(\textbf{x})-(p+2\delta)\cdot opt \big)+(p+2\delta)\cdot opt \\
& \leq e^{-\frac{f_2(\textbf{x}^\prime)-f_2(\textbf{x})}{opt}} \cdot \alpha^\prime_0e^{-\frac{f_2(\textbf{x})}{opt}}+(p+2\delta)\cdot opt \\
& = \alpha^\prime_0e^{-\frac{f_2(\textbf{x}^\prime)}{opt}}+(p+2\delta)\cdot opt.
\end{align*}
Hence, ${\bf x}'$  satisfies  $\pi^{(1)}$. Claim 1 is proved.


The next claim shows that the cost of an individual in the first phase is not too high.

\textbf{Claim 2}. For any individual $\textbf{x} $ added into the bin system in the first phase, we have
\begin{align}\label{eq0316-6}
f_2(\textbf{x})\leq opt\cdot \ln\frac{\alpha^\prime_0}{opt-\delta}.
\end{align}

Suppose $\textbf{x}\in B_i$ with $i\geq \left \lfloor\frac{(p+2\delta+1)opt}{\delta}\right \rfloor$. By Claim 1, ${\bf x}$ satisfies inequality \eqref{eq0602-12}. Combining \eqref{eq0602-12} with the observation that $f_1(\textbf{x})=i\delta\geq\left \lfloor\frac{(p+2\delta+1)opt}{\delta}\right \rfloor\delta \geq \left(\frac{(p+2\delta+1)opt}{\delta}-1\right)\delta$, and rearranging, Claim 2 follows.

The next claim estimates the expected time it takes for $I$  to drop from $\beta$ to some value less than $\left \lfloor\frac{(p+2\delta+1)opt}{\delta}\right \rfloor $. Here, we assume that $\beta\geq \left \lfloor\frac{(p+2\delta+1)opt}{\delta}\right \rfloor$, otherwise, we can skip the first phase of analysis and directly jump to the second phase.

{\bf Claim 3.} The expected time it takes for $I$ to decrease from $\beta $ to some value less than $\left \lfloor\frac{(p+2\delta+1)opt}{\delta}\right \rfloor $ is at most $e\big(\beta-\left \lfloor\frac{(p+2\delta+1)opt}{\delta}\right \rfloor+1\big) (1+\beta)n$.

Initially, $P=\{{\bf 0}\}$ and $I=\beta $. Note that
\begin{equation}\label{eq0603-1}
\mbox{$I$ does not increase over time.}
\end{equation}
In fact, consider an individual ${\bf x}\in B_I$. If ${\bf x}$ stays in the bin system, then $I$ does not decrease. If ${\bf x}$ is deleted from the bin system, it must due to the generation of an individual ${\bf x}'$ which is weakly better than ${\bf x}$.  Note that such ${\bf x}'$ advances ${\bf x}$, and thus can enter the bin system. So in this case, the new $I$ is at most $\frac{f_1({\bf x}')}{\delta}\leq \frac{f_1({\bf x})}{\delta}$.

Now, we estimate the expected time it takes for $I$ to decrease by at least $1$ in the first phase. Assume $I\geq \left \lfloor\frac{(p+2\delta+1)opt}{\delta}\right \rfloor $ and ${\bf x}\in B_I$. Let
\begin{align}\label{eq0606-5}
b_{\bf x}&= \arg\max_{v\in X\setminus S_{\bf x}}-\Delta_vf_1(S_{\bf x})/f_2(v),
\end{align}
where $S_{\bf x}$ is the set of elements corresponding to ${\bf x}$.  Let ${\bf x}'$ be the individual obtained from ${\bf x}$ by changing the bit corresponding to $b_{\bf x}$ from 0 to 1.

Note that
\begin{equation}\label{eq0811-3}
f_1({\bf x}')\leq f_1({\bf x})-\delta.
\end{equation}
Before proving \eqref{eq0811-3}, we first observe that as long as $f_1({\bf x})>0$, there always exists an element $v\in X\setminus S_{\bf x}$ such that $-\Delta_vf_1({\bf x})>0$. In fact, condition $(\romannumeral1)$ of Theorem \ref{thm1} guarantees that the element $v=\arg\max_{u\in X\setminus S_{\bf x}}\Delta_ug(S_{\bf x})$ satisfies $g(S_{\bf x}\cup \{v\})>g(S_{\bf x})$, and thus by the definition of $\delta$, we have $g(S_{\bf x}\cup \{v\})\geq g(S_{\bf x})+\delta$. It follows that the individual ${\bf y}$ corresponding to $S_{\bf x}\cup \{v\}$ satisfies
\begin{align*}
f_1({\bf x})-f_1({\bf y}) & =\delta\cdot\left(\left \lfloor\frac{g(X)-g({\bf x})}{\delta}\right \rfloor-\left \lfloor\frac{g(X)-g({\bf y})}{\delta}\right \rfloor\right)\\
& > \delta\cdot\left(\frac{g(X)-g({\bf x})}{\delta}-1-\frac{g(X)-g({\bf y})}{\delta}\right)\\
& =g({\bf y})-g({\bf x})-\delta\geq 0,
\end{align*}
where ``$>$'' holds by the observation that $a-1<\lfloor a\rfloor\leq a$ ($\forall a\in \mathbb R$). Then, by the choice of $b_{\bf x}$ in \eqref{eq0606-5}, we have $f_1({\bf x})-f_1({\bf x}')>0$.
Note that the $f_1$-value of an individual can only be a multiple of $\delta$, so $f_1({\bf x})-f_1({\bf x}')>0$ implies $f_1({\bf x})-f_1({\bf x}')\geq \delta$.

Next, we show that
\begin{equation}\label{eq0811-1}
\mbox{the above ${\bf x}'$ advances ${\bf x}$.}
\end{equation}
Suppose $C^*\setminus S_{\bf x}=\{v_1,\dots, v_t,\dots, v_{\hat{t}}\}$ is the decomposition described in Theorem \ref{thm1}. Similar to the proof of \eqref{eq0314-4}, we have $t\leq opt$.
Combining this with the choice of $b_{\bf x}$ in \eqref{eq0606-5}, Remark \ref{rem2}, the facts $f_1(S_{\bf x}\cup  C^*_{t})=0$ and $f_2(C_{t})\leq opt$, we have
\begin{align*}
\frac{-\Delta_{b_{\bf x}}f_1(S_{\bf x})}{f_2(b_{\bf x})}&\geq \frac{-\sum_{j=1}^{t}\Delta_{v_j}f_1(S_{\bf x})}{\sum_{j=1}^{t}f_2(v_j)} \nonumber \\
&\geq \frac{\sum_{j=1}^{t}(-\Delta_{v_j}f_1(S_{\bf x}\cup C^*_{j-1})-p-2\delta)}{opt}   \nonumber \\
&= \frac{\sum_{j=1}^{t}(-\Delta_{v_j}f_1(S_{\bf x}\cup C^*_{j-1}))-(p+2\delta)\cdot t}{opt} \nonumber \\
&\geq \frac{f_1(S_{\bf x})-f_1(S_{\bf x}\cup C^*_t)-(p+2\delta)\cdot opt}{opt}     \nonumber \\
&=\frac{f_1(S_{\bf x})-(p+2\delta)\cdot opt}{opt}
\end{align*}
Rearranging this inequality, using $f_2({\bf x}' )-f_2({\bf x} )=f_2(b_{\bf x} )$, we obtain inequality \eqref{eq0316-2-2}.
 Furthermore, by inequality \eqref{eq0811-3}, and because $f_2({\bf x}')=f_2({\bf x})+w(b_{\bf x})\leq f_2({\bf x})+w_{\max}$, individuals ${\bf x}'$ and ${\bf x}$ satisfy \eqref{eq0316-2}. So, ${\bf x}'$ advances ${\bf x}$.

As a consequence of \eqref{eq0811-1}, if ${\bf x}$ is mutated into ${\bf x}'$, then either ${\bf x}'$ or an individual ${\bf y}\in P$ with ${\bf y}\succ {\bf x}'$ can be added into the bin system. In this case, the tracker $I$ will be decreased to $\frac{f_1({\bf x}')}{\delta}$ or $\frac{f_1({\bf y})}{\delta}$, which are smaller than $I$.

By our synchronous maintaining of the bin system and the population, individual ${\bf x}$ belongs to $P$. The probability that ${\bf x}$ is picked by line \ref{line0606-1} of Algorithm \ref{algo2} is $1/|P|\geq 1/(\beta+1)$ (by Lemma \ref{lem0606-3}), and the probability that ${\bf x}$ is mutated into the above ${\bf x}'$ is $\left(\frac{1}{n}\right)\left(1-\frac{1}{n}\right)^{n-1}\geq \frac{1}{en}$. Hence, the probability
$$
\Pr(I \ \mbox{decreases by at least 1})\geq\frac{1}{e(\beta+1)n},
$$
and thus the expected time it takes for $I$ to decrease by at least 1 is at most $e(\beta+1)n$. Combining this with property \eqref{eq0603-1},  the total expected time it takes for $I$ to decrease from $\beta$ to some value less than $\left \lfloor\frac{(p+2\delta+1)opt}{\delta}\right \rfloor$ is upper bounded by $e\big(\beta-\left \lfloor\frac{(p+2\delta+1)opt}{\delta}\right \rfloor+1\big) (1+\beta
)n$. Hence, Claim 3 is proved.

\textbf{We next conduct the second phase of analysis.}  Assume that $I$ has dropped to some value $\gamma<\left \lfloor\frac{(p+2\delta+1)opt}{\delta}\right \rfloor$ at the end of the first phase analysis.
We say  an  individual ${\bf x}$ satisfies a quality-control condition $\pi^{(2)}$ if
\begin{align}\label{eq0666-1}
f_2(\textbf{x})\leq w_{\max}\left( \left \lfloor\frac{(1+p+2\delta)opt}{\delta}\right \rfloor-\frac{f_1(\textbf{x})}{\delta}\right)+\left(\ln\frac{\alpha^\prime_0}{opt- \delta}\right)opt.
\end{align}
Moreover,  we say that ${\bf x}'$ {\em advances ${\bf x}$ in the second phase} if either ${\bf x}'\succeq {\bf x}$ or they satisfy relation \eqref{eq0316-2}. A newly generated individual ${\bf x}'$ has a potential to enter the bin system if ${\bf x}'$ advances some existing individual in the {\em second-phase bin system} (that is, ${\bf x}'$ advances ${\bf x}\in B_j$ with $j<\left \lfloor\frac{(1+p+2\delta)opt}{\delta}\right \rfloor$). The manipulation on the bin system is similar to that in the first phase, using a different meaning of advance: if ${\bf x}'$ has the above potential, then add either ${\bf x}'$ or an individual ${\bf y}\in P$ with ${\bf y}\succ {\bf x}'$ into the bin system, depending on whether ${\bf x}'$ is eligible to enter $P$. And a consistency operation is executed to remove those inferior individuals from both $P$ and the bin system.

Assume that ${\bf x}_1$ is the first individual entering the second-phase bin system (it is the last individual added at the end of the first phase and ${\bf x}\in B_{\gamma}$). The next claim builds a bridge between the first and the second phases.

{\bf Claim 4.} ${\bf x}_1$ satisfies condition $\pi^{(2)}$, and any individual that is added to the second-phase bin system satisfies $\pi^{(2)}$.

  Because ${\bf x}_1$ is the first individual entering the second-phase bin system, ${\bf x}_1$ must advance some existing individual ${\bf x}$ in the first-phase bin system. By Claim 2, ${\bf x}$ satisfies inequality \eqref{eq0316-6}. Then due to the criteria of advance in the second phase, we have
$$
f_2({\bf x}_1)\leq f_2({\bf x})+w_{\max}< opt\cdot \ln\frac{\alpha^\prime_0}{opt- \delta}+ w_{\max}.
$$
Combining this with $\frac{f_1({\bf x}_1)}{\delta}=\gamma\leq\left \lfloor\frac{(1+p+2\delta)opt}{\delta}\right \rfloor -1$, individual ${\bf x}_1$ satisfies inequality \eqref{eq0666-1}. The first part of Claim 4 is proved.

We next prove the second part of Claim 4 by induction. Consider any individual ${\bf x}'\neq {\bf x}_1$ which is added to the second-phase bin system, then ${\bf x}'$ advances some existing individual ${\bf x}$ in the second-phase bin system. By induction, ${\bf x}$ satisfies  $\pi^{(2)}$, that is, inequality \eqref{eq0666-1}.  If ${\bf x}'\succeq {\bf x}$, then ${\bf x}'$ also satisfies inequality \eqref{eq0666-1}. If ${\bf x}$ and ${\bf x}'$ satisfy relation \eqref{eq0316-2}, then
\begin{align*}
f_2(\textbf{x}^\prime) &\leq f_2(\textbf{x})+ w_{\max}\nonumber\\&\leq  w_{\max} \left( \left \lfloor\frac{(p+2\delta+1)opt}{\delta}\right \rfloor-\frac{f_1(\textbf{x})}{\delta}\right)+\left(\ln\frac{\alpha^\prime_0}{opt-\delta}\right)opt+ w_{\max}\nonumber\\
&=w_{\max} \left( \left\lfloor\frac{(p+2\delta+1)opt}{\delta}\right \rfloor-\left(\frac{f_1(\textbf{x})}{\delta}-1\right)\right)+\left(\ln\frac{\alpha^\prime_0}{opt-\delta}\right)opt \nonumber \\
& \leq w_{\max}\left( \left \lfloor\frac{(p+2\delta+1)opt}{\delta}\right \rfloor-\frac{f_1(\textbf{x}^\prime)}{\delta}\right)+\left(\ln\frac{\alpha^\prime_0}{opt-\delta}\right)opt.
\end{align*}
In any case, ${\bf x}'$ also satisfies $\pi^{(2)}$. Claim 4 is proved.

 Note that if the first phase does not exist, that is, if $\frac{f_1({\bf 0})}{\delta}< \left \lfloor\frac{(p+2\delta+1)opt}{\delta}\right \rfloor $, then the initial solution ${\bf 0}$ satisfies condition $\pi^{(2)}$, and all the remaining arguments go through.

Similar to the derivation in Claim 3, we have the following claim that estimates the expected time it takes for $I$ to drop from  $\gamma $ to $0$.

{\bf Claim 5.} The expected time  it takes for the tracker $I$ to drop from $\gamma $ to $0$ is at most $e\left(\left \lfloor\frac{(p+2\delta+1)opt}{\delta}\right \rfloor-1\right) (1+\beta)n$.

{\bf Putting these two phases together:} the total expected time it takes for $I$ to decrease from $\beta $ to $0$ is at most $e\big(\left \lfloor\frac{(p+2\delta+1)opt}{\delta}\right \rfloor-1\big) (1+\beta)n+e\big(\beta-\left \lfloor\frac{(p+2\delta+1)opt}{\delta}\right \rfloor+1\big) (1+\beta)n=e\beta(1+\beta)n=O(\beta^2n)$. Note that the individual ${\bf x}$ in $B_0$ satisfies $f_1({\bf x})=0$ and condition $\pi^{(2)}$, hence,  ${\bf x}$ is a nearly feasible solution to the MinGC instance with $f_2({\bf x})\leq w_{\max}\left \lfloor\frac{(p+2\delta+1)opt}{\delta}\right \rfloor+\big(\ln\frac{\alpha^\prime_0}{opt-\delta}\big)opt\leq\big(\frac{w_{\max}}{\delta}\big(p+1+2\delta)+\ln\frac{\alpha^\prime_0}{opt-\delta}\big)opt$. The approximation ratio is proved.

 When $g(\cdot)$ is integer-valued, then by Remarks \ref{rem1} and \ref{rem2}, we can obtain improved results as claimed in the theorem.
\end{proof}

\subsection{Further Discussion}\label{sec0308-1}

According to Theorem \ref{thm1} and Theorem \ref{thm2}, GSEMO returns a nearly feasible solution with an approximation ratio that is comparable to that of the greedy algorithm. In particular, when $g(\cdot)$ is integer-valued, there is no violation of feasibility and the approximation ratio of GSEMO coincides with that achieved by the greedy algorithm. In the following, we consider some special cases.

For the {\em minimum submodular cover} (MinSubmC) problem, the utility function $g(\cdot)$ is submodular, which implies $p=0$. The trouble with a {\em real-valued} MinSubmC instance lies in the fact that the utility function $g(\cdot)$ might have too many different values to be manipulated efficiently by an evolutionary algorithm. Discretization is a natural choice to solve this problem. However, after discretization, submodularity is lost (see Remark \ref{rem2}). Nevertheless, this situation can be successfully dealt with using the condition formulated in Theorem \ref{thm1}, resulting in approximation ratio $\frac{w_{\max}}{\delta}\big(1+2\delta)+\ln\frac{ f_1({\bf 0})-2\delta\cdot opt}{opt-\delta}$.

For an {\em integer-valued} MinSubmC instance, by the second half of Theorem \ref{thm2}, GSEMO  obtains a feasible solution with an approximation ratio at most $w_{\max}+\ln\frac{g(X)}{opt}$. In particular, for an unweighted instance in which $w\equiv 1$, by the submodularity of $g$, for any optimal solution $C^*$, we have $g(X)= g(C^*)\leq \sum_{v\in C^*}g(v)\leq \max_{v\in C^*}g(v)\cdot |C^*|\leq \max_{v\in X}g(v)\cdot opt$. Hence the approximation ratio is at most $1+\ln \max\limits_{x\in X}g(x)$. The expected running time is $O(g(X)^2n)$.  This ratio matches the best one achieved by an approximation algorithm \cite{Ding-ZhuDu}.

The {\em minimum connected dominating set} (MinCDS) problem is another special case of the MinGC problem. Given a connected graph $G=(V,E)$, a vertex set $C$ is a  connected dominating set (CDS) if every vertex $v\in V\setminus C$ has at least one neighbor in $C$ and the subgraph of $G$ induced by $C$, denoted by $G[C]$ is connected. The goal of MinCDS is to find a  CDS with the minimum size. Taking $w\equiv 1$ and $g(C)=p(C)+q(C)- 2$, where $p(C)$ is the number of connected components of $G[C]$ and  $q(C)$ is the number of connected components of $G\langle C\rangle$ which is a spanning subgraph of $G$ induced by those edges incident with $C$, then MinCDS is a MinGC problem with $p=1$ and $g(X)=n-2$, where $n$ is the number of vertices. To make use of Theorem \ref{thm2}, a crucial observation is: if we order $C^*$ as $\{v_1^*,\ldots,v_t^*\}$ such that for any $i=1,\ldots,t$,
\begin{equation}\label{eq0309-1}
\mbox{the induced subgraph $G[\{v_1^*,\ldots,v_i^*\}]$ is connected}
\end{equation}
(notice that such an ordering exists since $G[C^*]$ is connected), then the condition described in Theorem \ref{thm1} is satisfied with $p=1$. In fact it can be proved that function $-q(\cdot)$ is submodular, and thus $-\Delta_{v_{i}}q  (C^*_{i-1}\cup C)\leq -\Delta_{v_{i}} q(C)$. However, $-p(C)$ is not submodular. In a worst case, $-\Delta_{v_{i}}p(C^*_{i-1}\cup C)$ can be larger than $-\Delta_{v_{i}}g(C)$ by the number of connected components in $G[C^*_{i-1}]$. Hence, under the ordering specified in \eqref{eq0309-1}, this gap can be bounded by $1$. As a result, GSEMO yields a CDS with an approximation ratio of at most $(2+\ln\frac{n-2-opt}{opt})$. Since $(n-2-opt)/opt\leq \delta_{\max}$, where  $\delta_{\max}$ is the maximum degree of the graph, so the approximation ratio is at most $2+\ln(\delta_{\max})$, which coincides with the one obtained by the approximation algorithm in \cite{Zhouj1}. Furthermore, because $\beta=n-2$, the expected running time is $O(n^3)$.

\section{Conclusion}\label{sec4}

This paper proposes a technique called  multi-phase bin-tracking analysis and we use this technique to analyze the performance bound of GSEMO for the MinGC problem. We show that for two important special cases of the MinGC problem, GSEMO yields approximation ratios matching those achieved by the greedy algorithm. Our analysis provides a valuable framework to help understand  how a greedy mechanism is embedded in an evolutionary algorithm. In fact, the key step for the bin-tracking analysis is to find out under which situation there exists an evolutionary path which is no worse than a greedy path, and the evolutionary process will not yaw.

It was worth mentioning that although we restrict our attention to the minimization problem, the proposed technique of multi-phase bin-tracking analysis can be easily modified to suit maximization problems too, examples of such problems include the maximum matroid base problem \cite{Qian1,Reichel,Zhoubook} and the maximum submodular cover problem \cite{Friedrich1,Qian3,Qian4}. In fact, a one-phase bin-tracking analysis works for these problems.

In the future, we would like to  find out more combinatorial optimization problems which can be solved approximately by an evolutionary algorithm. More importantly, we are interested in finding some common structural properties shared among those problems that lead to performance guarantees.

\section*{Acknowledgment}

This research is supported in part by National Natural Science Foundation of China (U20A2068, 11771013), Zhejiang Provincial Natural Science Foundation of China (LD19A010001).


\begin{thebibliography}{10}

\bibitem{DoerrBook}
Benjamin Doerr and Neumann Frank.
\newblock {\em Theory of evolutionary computation: recent developments in
  discrete optimization}.
\newblock Springer, Cham, Switzerland, 2020.

\bibitem{Ding-ZhuDu}
Ding-Zhu Du, Ker-I Ko, and Xiaodong Hu.
\newblock {\em Design and Analysis of Approximation Algorithms}.
\newblock Springer, New York, NY, New York, 2012.

\bibitem{Friedrich}
Tobias Friedrich, Jun He, Nils Hebbinghaus, Frank Neumann, and Carsten Witt.
\newblock Approximating covering problems by randomized search heuristics using
  multi-objective models.
\newblock {\em Evolutionary Computation}, 18(4):617--633, 2010.

\bibitem{Friedrich1}
Tobias Friedrich and Frank Neumann.
\newblock Maximizing submodular functions under matroid constraints by
  evolutionary algorithms.
\newblock {\em Evolutionary Computation}, 23(4):543--558, 2015.

\bibitem{Gao}
Wanru Gao, Tobias Friedrich, and Frank Neumann.
\newblock Fixed-parameter single objective search heuristics for minimum vertex
  cover.
\newblock In {\em Parallel Problem Solving from Nature -- PPSN XIV}, pages
  740--750, Switzerland, 2016. Springer, Cham.

\bibitem{Giel0}
Oliver Giel.
\newblock Expected runtimes of a simple multi-objective evolutionary algorithm.
\newblock In {\em The 2003 Congress on Evolutionary Computation (CEC'03)},
  volume~3, pages 1918--1925, Canberra, Australia, 2003. IEEE.

\bibitem{Levin}
Refael Hassin and Asaf Levin.
\newblock A better-than-greedy approximation algorithm for the minimum set
  cover problem.
\newblock {\em Siam Journal on Computing}, 35(1):189--200, 2005.

\bibitem{Kratsch}
Stefan Kratsch and Frank Neumann.
\newblock Fixed-parameter evolutionary algorithms and the vertex cover problem.
\newblock {\em Algorithmica}, 65(4):754--771, 2009.

\bibitem{NeumannBook}
Frank Neumann and Carsten Witt.
\newblock {\em Bioinspired computation in combinatorial optimization:algorithms
  and their computational complexity}.
\newblock Springer, Berlin, Heidelberg, Berlin, 2010.

\bibitem{Oliveto}
Pietro~S. Oliveto, Jun He, and Xin Yao.
\newblock Analysis of the $(1+1)$-ea for finding approximate solutions to
  vertex cover problems.
\newblock {\em IEEE Transactions on Evolutionary Computation},
  13(5):1006--1029, 2009.

\bibitem{Pourhassan2}
Mojgan Pourhassan, Wanru Gao, and Frank Neumann.
\newblock Maintaining 2-approximations for the dynamic vertex cover problem
  using evolutionary algorithms.
\newblock In {\em Genetic and Evolutionary Computation Conference 2015, ACM},
  pages 1--8, Madrid, Spain, 2015. ACM.

\bibitem{Pourhassan1}
Mojgan Pourhassan, Vahid Roostapour, and Frank Neumann.
\newblock Runtime analysis of rls and $(1+1)$ ea for the dynamic weighted
  vertex cover problem.
\newblock {\em Theoretical Computer Science}, 832:20--41, 2020.

\bibitem{Pourhassan}
Mojgan Pourhassan, Feng Shi, and Frank Neumann.
\newblock Parameterized analysis of multiobjective evolutionary algorithms and
  the weighted vertex cover problem.
\newblock {\em Evolutionary Computation}, 27(4):559--575, 2019.

\bibitem{Qian3}
Chao Qian, Jing-Cheng Shi, Yang Yu, Ke~Tang, and Zhi-Hua Zhou.
\newblock Parallel pareto optimization for subset selection.
\newblock In {\em Proceedings of the 25th International Joint Conference on
  Artificial Intelligence (IJCAI-16)}, pages 1939--1945, New York, 2016. AAAI
  Press.

\bibitem{Qian4}
Chao Qian, Yang Yu, Ke~Tang, Xin Yao, and Zhi-Hua Zhou.
\newblock Maximizing submodular or monotone approximately submodular functions
  by multi-objective evolutionary algorithms.
\newblock {\em Artificial Intelligence}, 275:279--294, 2019.

\bibitem{Qian1}
Chao Qian, Yang Yu, and Zhi-Hua Zhou.
\newblock On constrained boolean pareto optimization.
\newblock In {\em Proceedings of the 24th International Joint Conference on
  Artificial Intelligence (IJCAI-15)}, pages 389--395, Buenos Aires Argentina,
  2015. AAAI Press.

\bibitem{Qian2}
Chao Qian, Yang Yu, and Zhi-Hua Zhou.
\newblock Subset selection by pareto optimization.
\newblock In {\em Proceedings of the 28th International Conference on Neural
  Information Processing Systems (NIPS'15)}, pages 1774--1782, Montreal Canada,
  2015. MIT Press.

\bibitem{Reichel}
Joachim Reichel and Martin Skutella.
\newblock Evolutionary algorithms and matroid optimization problems.
\newblock {\em Algorithmica}, 57(1):187--206, 2010.

\bibitem{ShiF}
Feng Shi, Frank Neumann, and Jianxin Wang.
\newblock Runtime performances of randomized search heuristics for the dynamic
  weighted vertex cover problem.
\newblock {\em Algorithmica}, 83(4):906--939, 2021.

\bibitem{Vikhar}
Pradnya~A. Vikhar.
\newblock Evolutionary algorithms: A critical review and its future prospects.
\newblock In {\em 2016 International Conference on Global Trends in Signal
  Processing, Information Computing and Communication (ICGTSPICC)}, pages
  261--265, Jalgaon, India, 2016. IEEE.

\bibitem{Yu}
Yang Yu, Xin Yao, and Zhi-Hua Zhou.
\newblock On the approximation ability of evolutionary optimization with
  application to minimum set cover.
\newblock {\em Artificial Intelligence}, 180-181:20--33, 2012.

\bibitem{Zhouj1}
Jiao Zhou, Zhao Zhang, Weili Wu, and Kai Xing.
\newblock A greedy algorithm for the fault-tolerant connected dominating set in
  a general graph.
\newblock {\em Journal of Combinatorial Optimization}, 28(1):310--319, 2014.

\bibitem{Zhoubook}
Zhi-Hua Zhou, Yang Yu, and Chao Qian.
\newblock {\em Evolutionary Learning: Advances in Theories and Algorithms}.
\newblock Springer, Singapore, 2019.

\end{thebibliography}

\end{document}